\documentclass{article}
\usepackage{ijcai25}

\usepackage{fancyhdr}
\fancypagestyle{ijcaiproc}{%
  \fancyhf{}%
  \fancyhead[C]{\footnotesize Proceedings of the Thirty-Fourth International Joint Conference on Artificial Intelligence (IJCAI-25)}%
  \fancyfoot[C]{\small\thepage}%
  }
\fancypagestyle{ijcaifirst}{%
  \fancyhf{}%
  \fancyhead[C]{\footnotesize Proceedings of the Thirty-Fourth International Joint Conference on Artificial Intelligence (IJCAI-25)}%
  \fancyfoot[C]{\small\thepage\\[3pt]\scriptsize \copyright\ 2025 International Joint Conferences on Artificial Intelligence (IJCAI). Published in Proceedings of IJCAI-25, pp.~4806--4814. \url{https://doi.org/10.24963/ijcai.2025/535}}%
  }
\AtBeginDocument{\pagestyle{ijcaiproc}\setcounter{page}{4806}}

\usepackage{times}
\usepackage{soul}
\usepackage{url}
\usepackage[hidelinks]{hyperref}
\usepackage[utf8]{inputenc}
\usepackage[small]{caption}
\usepackage{graphicx}
\usepackage{amsmath}
\usepackage{amsthm}
\usepackage{booktabs}
\usepackage[switch]{lineno}
\usepackage{booktabs}
\usepackage{multirow}
\usepackage{cleveref}

\usepackage[T1]{fontenc}
\usepackage[noend,ruled,vlined,noline]{algorithm}
\usepackage{algpseudocode}
\usepackage{booktabs}
\usepackage[misc]{ifsym}
\usepackage{xcolor}
\usepackage{array}
\usepackage{xcolor}
\usepackage{color, colortbl}
\usepackage{mdwmath}
\usepackage{mdwtab}
\usepackage{eqparbox}
\usepackage{url}
\usepackage{amsmath,amssymb,amsfonts,mathrsfs}
\usepackage{amsthm}

\newtheorem{theorem}{Theorem}
\newtheorem{proposition}{Proposition}

\newcommand{\mc}{\mathcal}

\title{Grounding Methods for Neural-Symbolic AI}
\author{
Rodrigo	Castellano Ontiveros$^1$\and
Francesco Giannini$^2$\and
Marco Gori$^1$\and\\
Giuseppe Marra$^3$\And 
Michelangelo Diligenti$^1$
\affiliations
$^1$University of Siena, Italy\\
$^2$Scuola Normale Superiore, Italy\\
$^3$KU Leuven, Belgium\\
\emails
\{rodrigo.castellano,marco.gori,michelangelo.diligenti\}@unisi.it,\\
francesco.giannini@sns.it,
giuseppe.marra@kuleuven.be
}

\begin{document}

\maketitle
\thispagestyle{ijcaifirst}

\begin{abstract}
A large class of Neural-Symbolic (NeSy) methods employs a machine learner to process the input entities, while relying on a reasoner based on First-Order Logic to represent and process more complex relationships among the entities. A fundamental role for these methods is played by the process of logic grounding, which determines the relevant substitutions for the logic rules using a (sub)set of entities. Some NeSy methods use an exhaustive derivation of all possible substitutions, preserving the full expressive power of the logic knowledge. This leads to a combinatorial explosion in the number of ground formulas to consider and, therefore, strongly limits their scalability. Other methods rely on heuristic-based selective derivations, which are generally more computationally efficient, but lack a justification and provide no guarantees of preserving the information provided to and returned by the reasoner. 
Taking inspiration from multi-hop symbolic reasoning, this paper proposes a parametrized family of grounding methods generalizing classic Backward Chaining. Different selections within this family allow us to obtain commonly employed grounding methods as special cases, and to control the trade-off between expressiveness and scalability of the reasoner.
The experimental results show that the selection of the grounding criterion is often as important as the NeSy method itself.
\end{abstract}

\section{Introduction}
Neural-Symbolic (NeSy) methods~\cite{d2009neural,marra2024statistical} integrate the strengths of neural networks and symbolic reasoning, and heavily rely on the process of logic grounding, which bridges the semantic gap between neural representations and symbolic logic. 
However, classical grounding techniques in NeSy are impractical, and they do not take advantage of the generalization capabilities of NeSy approaches, which can deduce true facts using statistical methods without the need for a classical proof.
In fact, the selection of appropriate grounding represents a significant challenge for NeSy methodologies. Existing approaches exhibit a dichotomy between exhaustive derivations that maintain full expressiveness of logical knowledge but suffer from scalability issues, and heuristic-based selective derivations that sacrifice logical coherence for computational efficiency. Indeed, the employed grounding process is often completely neglected in the NeSy literature, or relegated to a single sentence in the appendix. This hinders the possibility to reproduce, compare, and even fully understand the strengths and weaknesses of the different methods.

This paper aims to fill this gap by exploring diverse grounding methods in NeSy AI. 
We contend that the selection of a suitable grounding criterion is as pivotal as the design of the NeSy method itself. To support this claim, we propose a parameterized generalization of the classic Backward Chaining algorithm, which leverages the well-studied grounding approaches employed in logic inference, but it can be relaxed to take advantage of the statistical nature of NeSy approaches. 
This allows the control of the trade-off between scalability and preservation of the logical context.
Experimentally, we show the efficacy of the proposed grounding methodologies across various NeSy methods on the link prediction task in Knowledge Graphs (KGs). 

\paragraph{Contributions.} The main contributions of this paper are: (i) we formalize a variety of parametrized grounding methods for NeSy models. The parametrization controls the trade-off between scalability, ensuring the applicability of NeSy methods to large datasets, and expressiveness, hence allowing the model to take full advantage of the knowledge; (ii) we show that our results are general and not specific to a single NeSy method, and we provide an extensive experimental evaluation to support this claim, (iii) we redefine a large class of popular NeSy methods within a common message passing schema, allowing us to share a common implementation and reuse the same grounding and scaling techniques.

\section{Preliminaries}
\label{sec:back}
\paragraph{First-Order Logic (FOL).} A function-free FOL language~\cite{barwise1977introduction} is defined by a set of constants (entities) $\mc{C}$, variables $\mathcal{V}$ and predicates (relations) $\mathcal{R}$. An atomic formula, or atom, is $r(a_1,...,a_k)$ with $a_1,...,a_k\in \mc{C} \cup \mathcal{V}$, $r\in \mathcal{R}$ and $k=arity(r)$. A literal is an atom or its negation, and each compound FOL formula $\alpha$ is obtained by recursively combining atomic formulas with connectives ($\neg,\land,\lor,\rightarrow,\leftrightarrow$) and quantifiers ($\forall$ and $\exists$). Examples of FOL formulas are $\forall x Nation(x) \rightarrow \exists y CapitalOf(y,x)$, expressing that  
``All nations have a capital'' or $\forall x\forall y\forall z LocIn(x,y)\wedge LocIn(y,z)\rightarrow LocIn(x,z)$, expressing the transitiveness of being located in a place. 
Given a formula $\alpha$, we define $\textit{vars}(\alpha)$ as the set of variables contained in $\alpha$. For simplicity, in the following sections we will focus on FOL theories, i.e. sets of formulas, containing only specific formulas called Horn clauses, which play a fundamental role in computational logic. A Horn clause is a disjunction of literals with a single positive literal, which is equivalent to formulas in the form $b_1 \land \ldots \land b_n \rightarrow h$, where the atoms $b_1, \ldots, b_n$ are called \textit{body} atoms and the single atom $h$ is called \textit{head}.

\paragraph{Grounding FOL Theories.} Let us fix a FOL language. A substitution $\theta$ is an assignment from variables to domain constants, like $\theta = \{x:\textit{Italy},  y:\textit{Europe}\}$. A \textit{ground formula} is a FOL formula where all the variables are constants. For example, given $\alpha=LocIn(x,y)$, we can \textit{ground} it with $\theta$ and obtain $\theta\alpha=LocIn(\textit{Italy},\textit{Europe})$.
The \textit{Herbrand Base} $HB = \{r(c_1, ..c_k) | r \in \mathcal{R}, k=arity(r), (c_1, ..c_k) \in \mc{C}^k\}$ is the set of all possible ground atoms formed using the predicate and constant symbols. Given a FOL theory $\Gamma$, its \textit{Herbrand Universe} is defined as $HU_{\Gamma} = \{\theta\alpha: \alpha\in\Gamma,\theta\mbox{ is a substitution}\}$,
i.e. $HU_{\Gamma}$ is the set of all possible ground formulas that can be formed from $\Gamma$.
The process of constructing the entire Herbrand Universe is called \textit{(full) grounding}. In general, we refer to \textit{grounding} as the process of constructing a subset of the Herbrand Universe for a given FOL Theory. 

Some methods like Markov Logic Networks (MLNs) \cite{richardson2006markov}, use a graphical representation of the entire Herbrand Universe (or a subset) called a \textit{Grounded Markov Network} (GMN). In the GMN, each node represents a ground atom, and there is an edge between two nodes only if the corresponding ground atoms appear together in at least one grounding of a formula (\Cref{fig:ex_GMN}).

\begin{figure}[t!]
    \centering   \includegraphics[width=1\linewidth]{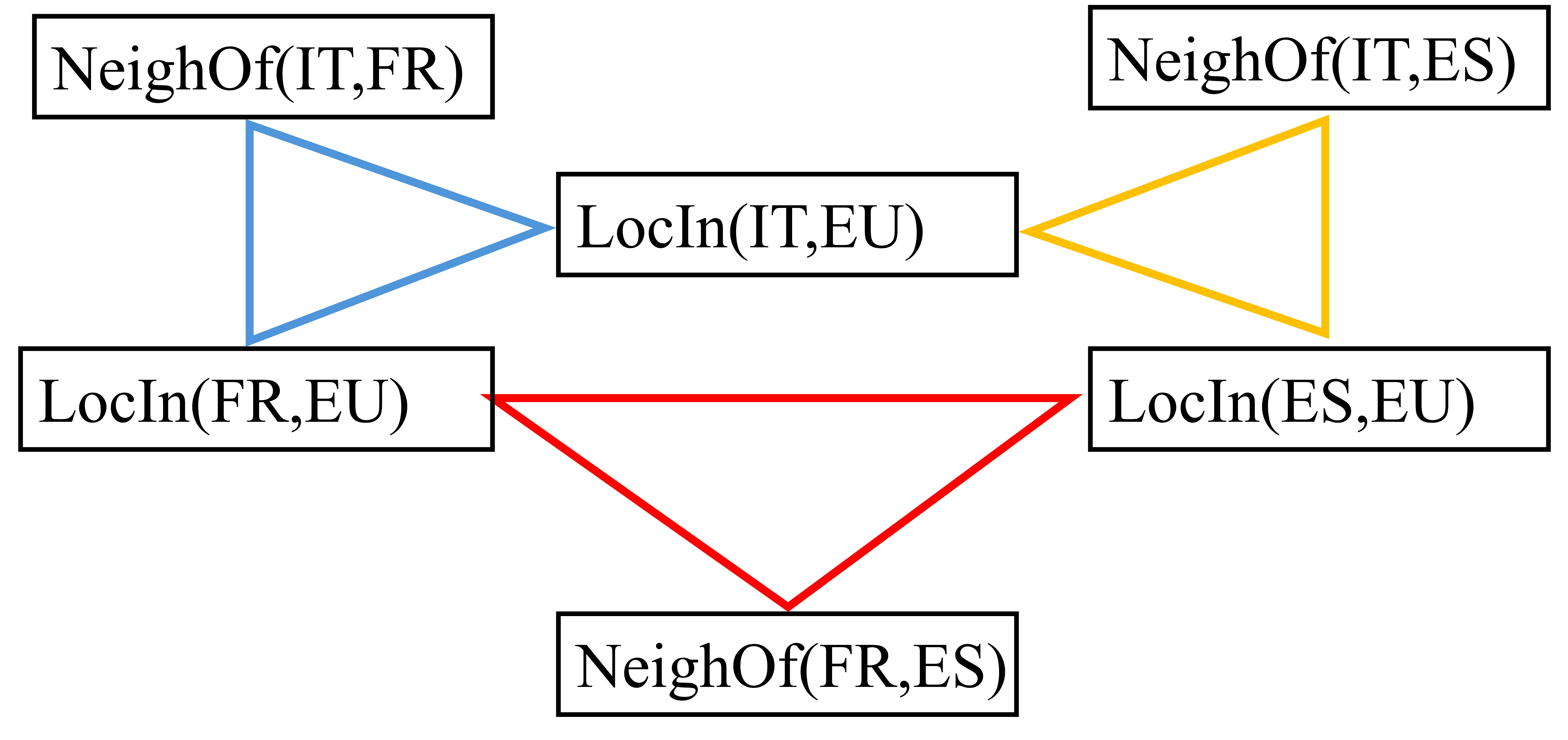}
    \caption{Portion of a GMN. We used the same color for atoms occurring in the same ground formula. The FOL theory includes the constants Italy (IT), France (FR), Spain (ES) and Europe (EU), the predicates LocatedIn (LocIn) and NeighbourOf (NeighOf), and the rule $\forall x\forall y\forall z~
    LocIn(x, z) \land NeighOf(x, y) \rightarrow LocIn(y, z)$.}
    \label{fig:ex_GMN}
\end{figure}

\paragraph{Knowledge Graph Embeddings.} A Knowledge Graph (KG) represents knowledge as a graph of entities (nodes) and relations (edges), with facts as (subject, relation, object) triples. Knowledge Graph Embeddings (KGEs)~\cite{wang2017knowledge} map KG entities and relations to low-dimensional vectors, encoding semantic meaning and statistical dependencies. KGEs enable effective reasoning even with incomplete/noisy graphs, for tasks like link prediction and query answering. Prominent examples include DistMult~\cite{yang2015embedding} and ComplEx~\cite{trouillon2016complex}.

\section{Neural-Symbolic Models}
\label{sec:nesy}

Neural networks, while effective at feature processing, are limited by opaque decision-making, large data requirements, and poor generalization. Symbolic AI, though interpretable, struggles with real-world complexity. Neuro-symbolic (NeSy) models combine their strengths. This paper focuses on NeSy methods using a neural network to process FOL constants/predicates, using the GMN topology to define the reasoner.

\begin{figure*}[th]
    \centering
\includegraphics[width=1\textwidth]{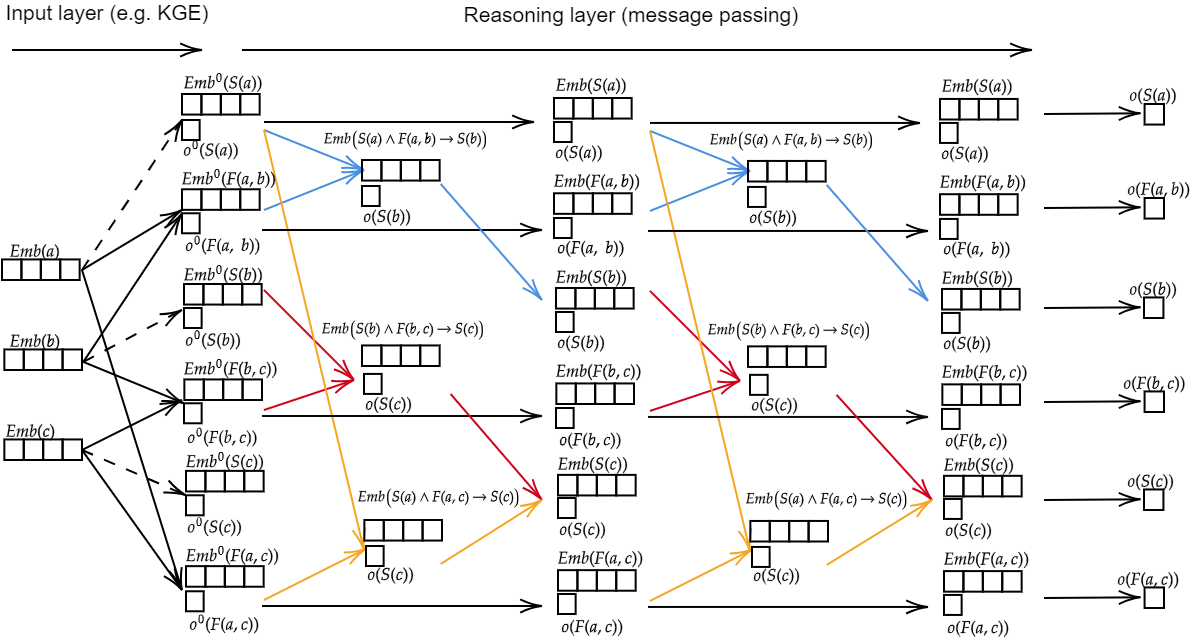}
\caption{In the considered NeSy models an atom processor takes as input the representation of the constants and compute an embedding and an initial prediction of the atoms. The representations and predictions are refined via message passing, thus performing the reasoning process.}
    \label{fig:model} 
\end{figure*}
The overall structure of these methods, inspired by \cite{barbiero2024relational}, is shown in \Cref{fig:model}, and it is composed of two high-level components:
\begin{enumerate}
\item \textbf{Atom Processing Layer} mapping each ground atom into an embedding and computing its initial prediction.
\item \textbf{Reasoning Layers} updating the atoms' embedding and score according to the structure imposed by the GMN associated to the logic theory, and computing a final updated version of the predictions for each ground atom. 
\end{enumerate}

\subsection{Atom Processing Layer}
The class of considered neural-symbolic methods assigns a numeric representation and an initial output prediction to all the ground atoms in the HB.
This can be formalized as: given a ground atom $\alpha = P(c_1,\ldots,c_n)$ and $Emb(c_i)$ the embedding of $c_i\in \mc{C}$,
the atom processing layer outputs the embedding $Emb^0(\alpha)$ and the initial output prediction $o^0(\alpha)$ as:
\begin{eqnarray*}
    Emb^0(\alpha) \!\!\!&=&\!\!\!\! Enc_P(Emb(c_1), ..., Emb(c_n))\\
    o^0(\alpha) \!\!\!&=&\!\!\! Score( Emb^0(\alpha) )
\end{eqnarray*}
where $Enc_P$ is a generic encoding function associated with the predicate $P\in\mc{R}$, 
and $Score$ is a generic scoring function.
Many different representation learning approaches can be employed to implement and co-train the functions $Enc, Score$. While this paper focuses on KGEs and their scoring functions, the approach can be trivially applied to pattern recognition problems using neural networks.

\subsection{Reasoning Layers}
The reasoning layers take the outputs of the atom processing layer and provide an updated representation (both as embedding and prediction) for each ground atom in the head of the rules. Finally, it outputs the predictions for all atoms. 
In the following, we indicate the set of rules of a FOL theory having an atom $h$ as head in the considered FOL grounding as $\mathcal{R}(h)$.
The reasoner can be seen as the unfolding network of a message passing process~\cite{gilmer2017neural} over the GMN, where messages flow from atoms appearing in the body of a ground formula to the atom in the head of the same ground formula. Below, we provide the message passing equations for some of the methods tested in this paper.

The messages flow from an atom node $b$ in the body of the $j$-th formula $r_j$ to the atom node $h$ in the head of the same ground formula as:
\begin{eqnarray*}
    Emb_j(h) &=& F_{j}( M_{ne(h,1) \rightarrow h}, ..., M_{ne(h,d_j) \rightarrow h}) \\
    o_j(h) &=& G_{j}( M_{ne(h,1) \rightarrow h}, ..., M_{ne(h,d_j) \rightarrow h})
\end{eqnarray*}
where $M^j_{b \rightarrow h}$ is the message defining the different neural-symbolic methods, $F_j,G_j$ are generic functions aggregating back the encoded representation and the output of a node, $ne(h,1), \ldots, ne(h,d_j)$ are the neighbours of node $h$ for the $j$-th rule, i.e. the set of body atoms associated to atom $h$ for the rule $r_j$, with $d_j = |ne(h)|$ the number of atoms occurring in the body of the $j$-th rule $r_j$.

The embedding of an atom node $h$ is updated by aggregating the representation for all the rules having atom $h$ as head:
$Emb(h) = \underset{r_j \in \mathcal{R}(h)}{\mathcal{A}^e} \ Emb_j(h),
$
where $\mathcal{A}^e$ is an aggregation operator like sum, mean or max.
Other methods instead aggregate directly over the outputs of a node:
$o(h) = \underset{r_j \in \mathcal{R}(h)}{\mathcal{A}^o} \ o_j(h), $
where $\mathcal{A}^o$ can be any aggregation operator, even if most methods set it as a disjunctive operator, like a max, as shown below.

\paragraph{Relational Reasoning Networks (R2N).} R2N \cite{marra2025relational} propagates the embeddings as representations of the atoms, using the following message passing schema:
\begin{eqnarray*}
    M_{b \rightarrow h} &=& Emb(b) \\
    Emb_j(h) &=& MLP^e_{j}( M_{ne(h,1) \rightarrow h}, ..., M_{ne(h,d_j) \rightarrow h}) \\
    Emb(h) &=& \displaystyle\sum_{r_j \in \mathcal{R}(h)} Emb_j(h) \\
    o(h) &=& MLP^o(Emb(h))
\end{eqnarray*}
where the embeddings are initialized using $Emb^0$, and $MLP^e_{j},MLP^o$ are multi-layer perceptrons (MLP) processing the embeddings and computing the outputs, respectively.

\paragraph{Logic Tensor Networks (LTN).} LTN~\cite{badreddine2022logic} and Semantic Based Regularization (SBR)~\cite{diligenti2017semantic}
aggregate the predictions of the body atoms of the $j$-th rule using a selected t-norm and enforce the consistency of the predictions with the logic rule. These methods were limited to unary predicates in the original papers, but here we propose a relational extension based on message passing:
\begin{eqnarray*}
    M_{b \rightarrow h} &=& o(b) \\
    o_j(h) &=& t{-}norm(M_{ne(h,1) \rightarrow h}, \ldots, M_{ne(h,d_j) \rightarrow h}) \nonumber \\
    o(h) &=& \displaystyle\max_{r_j \in \mathcal{R}(h)} o_j(h)
\end{eqnarray*}
where $o(\alpha)$ is initialized using $o^0(\alpha)$ and the original formulation was running a single propagation step.

\paragraph{Deep Concept Reasoners (DCR).} DCR \cite{barbiero2023interpretable} builds a formula from a set of candidate atoms before computing the output of a formula using a t-norm:
\begin{eqnarray*}
   M_{b \rightarrow h} &=& \Psi_j(Emb^0(b), \Phi_j(Emb^0(b), o(b))) \\
   o_j(h) &=& t{-}norm\left( M_{ne(h,1) \rightarrow h}, \ldots, M_{ne(h,d_j) \rightarrow h} \right) \\
   o(h) &=& \displaystyle\max_{r_j \in \mathcal{R}(h)} o_j(h)
\end{eqnarray*}
where $o_j(\alpha)$ is initialized using $o^0_j(\alpha)$, and the functions
$\Phi_j: \mathbb{R}^{n+1} \rightarrow [0, 1]$ and
$\Psi_j: \mathbb{R}^{n+1} \rightarrow [0, 1]$, assuming $Emb^0(\alpha)\in\mathbb{R}^n$, process the embedded representation of each atom in each rule $j$, to compute its polarity and relevance values, respectively, to get a learned Horn Clause. 
The process can possibly be iterated, even if the original formulation is defined for a single step of propagation.

\section{Grounding Methods}
\label{sec:grounders}

\begin{figure}[t]
  \centering
  \includegraphics[width=0.43\textwidth]{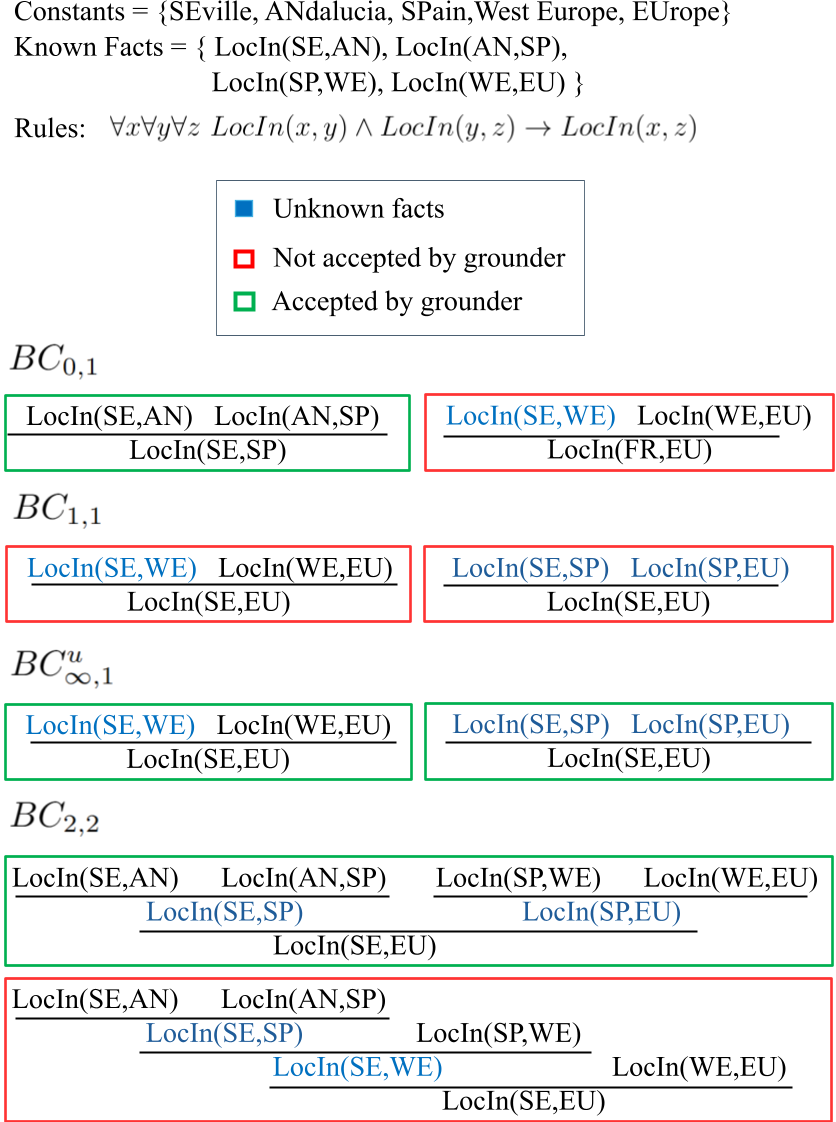}
  \caption{Illustration of some accepted/rejected proofs for different queries with respect to $w,d$ ($BC_{w,d}$) parameters.}
  \label{fig:backward_chaining}
\end{figure}

Logical reasoning is often used as a form of theorem proving~\cite{loveland2016automated} in computational logic. An example is query answering in Prolog, where the answer to a query is computed by searching for supporting facts given the set of rules of the program. Hence, proving can be cast to a search over the GMN, looking for subtrees (i.e., proofs) that connect the query to a set of supporting known facts.
The known facts in a logic theory are typically much smaller than the HB.
Classic Backward Chaining in Prolog considers unknown facts as possibly provable, and recursively builds a sub-proof for them if they are the head of a logic rule. While this process is effective in finding a formal proof for a query given the available known facts, it can be very time-consuming. A commonly used strategy consists in fixing a maximum depth for the backward search, thus admitting to find only proofs within a maximum number of reasoning steps. %

In NeSy, the logic facts can be neither true nor false, as they are assigned a score (e.g. a probability of being true) by the corresponding input layer. Therefore, it would be limiting to only consider known true facts in the proving process. Many other facts can provide uncertain information (i.e. context) to predict a query and should be considered. Unfortunately, the size of the portion of the ground network relevant to the query may grow exponentially.
However, such a limitation also introduces a new opportunity. Since NeSy models can score any fact in the knowledge base, one does not necessarily need to search the supporting known facts, as the scored facts at shallow depths in a proof can provide enough context for the learning to happen at a statistical level. Nevertheless, the majority of existing NeSy models completely neglect such opportunity, requiring very deep proving processes and resulting in intractable methods for inference. 

\subsection{Parameterized Backward Chaining Grounder}
The discussed limitations of classical Backward Chaining for NeSy motivate a novel definition of a flexible parameterized class of Backward Chaining grounders, which may take advantage of the generalization capabilities of NeSy methods over partially proved trees. 
A grounder in this parametrized class is defined as $BC_{w,d}$, where $w,d\in\mathbb{N}$ are referred to as the \textbf{width} and \textbf{depth} of the grounder, respectively.
The width $w$ determines the maximum number of atoms in the body of each ground rule that is allowed to not be in the training data and, in turn, to be proved as sub-goals. The depth $d$ controls the maximum depth that can be achieved by any proof for any query passing through its sub-goals.
If the number of unknown ground atoms is bigger than $w$ in any ground rule, the proof is immediately discarded even if more depth is allowed.
Moreover, once the proof has been built by the grounder, we allow two options: i) to discard the proof if it contains unknown atoms (as done by Backward Chaining), ii) to accept the proof by using the scores of the unknown ground atom. We denote these grounders allowing uncertain atoms as $BC^u_{w,d}$. For simplicity, if not stated differently, we will assume option i).
When using option i), $BC_{0,1} \equiv BC_{1,1}$, so we will use $BC_{0,1}$ in the rest of the paper.

To illustrate the effects of depths and widths, \Cref{fig:backward_chaining} compares the acceptance of diverse proofs across different parameter values.
\Cref{sec:exp} discusses the experimental results for different versions of $BC_{w,d}$.
Here below, we show how $BC_{w,d}$ subsumes as special cases  grounding techniques commonly used in the literature. 

\paragraph{Backward Chaining Grounder  $(w=\infty, d=\infty)$.} Backward Chaining is an inference strategy used in logic programming and automated theorem provers to efficiently prove unknown facts.
Given a query to prove, it searches for a rule, whose head atom matches the query. If some atom in the body of the rule is not in the training data, it becomes a new sub-goal that needs to be proven. The process is repeated recursively, until the query is proved with an associated proof tree.
Notice that, for $w=d=\infty$ we mean $w=B$, being $B$ the maximum amount of atoms in the body of all the rules.
This is the grounding strategy employed by DeepProbLog \cite{manhaeve2018deepproblog}, but its application is unfeasible at the scale of the learning tasks considered in this paper.

\paragraph{Backward Chaining Grounder with limited depth $(w=\infty, d=n)$.} In practical settings, the growth of the admissible proofs is limited, as often done in logic programming. This can be accomplished by setting $d=n$, where $n>0$ determines the maximum depth of the proofs.

\paragraph{Known Body Grounder  $(w=0, d=1)$.}
Several popular NeSy models, like e.g. ExpressGNN \cite{zhang2020efficient} and RNNLOGIC \cite{qu2020rnnlogic}, accept a grounding only if all body atoms are known facts, and recursion is limited to depth 1. Therefore, these approaches are equivalent to $BC_{0,1}$.

\paragraph{Full Grounder $(w=\infty, d=1)$.}  Full grounding of a FOL theory $\Gamma$ involves generating the $HU_\Gamma$, and it is used by models like LTN, SBR or MLNs. %
This grounder assumes option ii) for proofs, i.e. proofs are accepted even with unknown ground atoms, and thus it is denoted as $BC^u_{\infty,1}$. 
Roughly, the Full Grounder considers each atom in the $HB$ as a possible query. Hence, all possible ground atoms are represented in the graph, often including irrelevant information, but at the same time ensuring that all the required information needed to perform a correct inference could be recovered. 

\paragraph{Complexity Analysis.} The performances of $BC_{w,d}$ for NeSy applications are bound by two conflicting effects. On one hand, a grounder should be able to find the proofs needed to prove the queries. In this regard, increasing the width and the depth will result in extending the number of admissible proofs, thus increasing the number of provable queries.
\begin{proposition}
The set of facts that are provable via $BC_{w,d}$ is a subset of the provable facts via $BC_{w,d+1}$ and of the provable facts via $BC_{w+1,d}$.
\end{proposition}
\begin{proof}
    The proof trivially follows by observing that the set of buildable proofs in $BC_{w,d}$ is monotone w.r.t. $w$ and $d$. Indeed, any proof with maximum width $w$ (or depth $d$) is also a proof with maximum width $w+1$ (or depth $d+1$).   
\end{proof}
However, a grounder should not instantiate more nodes than needed, as this may affect the generalization capabilities, as stated by the following theorem.
\begin{theorem}~\cite{scarselli2018vapnik}
\label{th:VC}
The Vapnik–Chervonenkis dimension (VCD) of the function $\phi$ computed by a GNN with sigmoidal outputs on the domain of graphs with up to $N$ nodes satisfies: 
$VCD(\phi) = O(p^2\cdot N^2)\ ,$
where $p$ is the number of parameters of the GNN.
\end{theorem}
The Vapnik–Chervonenkis theory states that the generalization error grows as the square root of the $VCD$, therefore the error bound grows linearly with the graph size.
The number of nodes of the underlying reasoning graph built by $BC_{w,d}$ increases as $O(G^dw^{d-1}m)$, where $G$ is the number of groundings with at most $w$ atoms not in the known KG, and $m$ is the maximum number of body atoms within any rule.
For instance, if the rule with the longest body is a chain in the form of $p(x_1,x_2)\land \ldots \land p(x_{n-1},x_n)\rightarrow p(x_1,x_n)$ with $n$ representing the number of variables, the number of groundings grows as: $G \propto C^{min(w, n-2)}$ for a single reasoning step, where $C$ is the number of constants, and the overall ground network size grows as $O(C^{min(w,n-2)\cdot d}w^{d-1}m)$.
\Cref{th:VC} shows that the generalization error is bound to grow linearly with the size of the grounded network. Higher values for $w$ or $d$ lead to an increase in the size of the GMN, therefore it is fundamental for the success of the NeSy approach to select values for $w,d$ such that the logical context is preserved, while the generalization capabilities of the GNN are not hindered.

\section{Experiments}
\label{sec:exp}
Several experiments have been performed on KG link prediction to compare the effectiveness of the grounders.\footnote{\url{https://github.com/rodrigo-castellano/Grounding-Methods}}

\paragraph{Datasets.} The Countries \cite{bouchard2015approximate}, Kinship \cite{kok2007statistical}, WN18RR \cite{dettmers2018convolutional} and FB15k-237 \cite{dou2021novel} datasets are used to capture diverse sizes and complexities. Countries is split into three tasks S\{1,2,3\} of increasing complexity, predicting country location based on regions and neighborhoods. Kinship represents familial relationships.
Dataset statistics are in \Cref{tab:symbolic_datasets}.

\paragraph{Rules.} The rules associated with the Countries dataset are $R_1=LocIn(x,w) \land  LocIn(w,z) \rightarrow LocIn(x,z)$, $R_2=NeighOf(x,y) \land LocIn(y,z) \rightarrow LocIn(x,z)$ and $R_3=NeighOf(x,y) \land NeighOf(y,k) \land LocIn(k,z) \rightarrow LocIn(x,z)$, where $LocIn$ refers to the predicate LocatedIn and $NeighOf$ refers to NeighbourOf. 
The task $S_1$ takes $R_1$, $S_2$ takes $R_1,R_2$, and $S3$ all of them. The Countries\_abl dataset is based only on $R_2$.
Rules for the remaining datasets were extracted using AMIE [Galárraga et al., 2013], selecting the top rules based on confidence. The number of selected rules per dataset is shown in \Cref{tab:symbolic_datasets}.

\begin{table}[t]
\centering
\label{tab:symbolic_datasets}
{\small
\setlength{\tabcolsep}{2pt}
\begin{tabular}{l|ccccc}
{\bf Dataset} & \#{\bf Entities} & \#{\bf Relations} & \#{\bf Facts} & \#{\bf Degree} & \#{\bf Rules}\\
\hline
Countries S1 & 272 & 3 & 1,110 & 4.28 & 1\\
Countries S2 & 272 & 4 & 1,062 & 4.35 & 2\\
Countries S3 & 272 & 4 & 978 & 4.35 & 3\\
Kinship      & 104 & 25 & 28,356 & 272.7 & 48 \\
WN18RR       & 40,943 & 18 & 93,003  & 2.3 & 17 \\
{FB15k-237} & {14505} & {237} & {310079} & {21.4} & 500 \\
\end{tabular}
}
\caption{Basic statistics for the employed datasets.}
\end{table}

\paragraph{Hyperparameters.}  The KGE uses a fixed embedding size of 100, with overall 1000 head and tail corruptions for WN18RR. Adam optimizer ($10^{-2}$ learning rate) and binary cross-entropy loss were used over 100 epochs.

\paragraph{Models.}  ComplEx is used as the KGE baseline and as the input layer for NeSy methods to enable a reasoning ablation study. The KGE baseline is compared with NeSy models of varying characteristics: R2N, which is highly expressive but less interpretable as it learns rules over groups of atoms, excels in complex, loosely formalized reasoning but risks overfitting when flexibility is unnecessary. SBR, less expressive but highly interpretable, performs best in simpler, well-formalized tasks as it strictly applies logic rules. DCR serves as a middle ground between these two approaches.

\paragraph{Metrics.}  For the evaluation, we use Mean Reciprocal Rank (MRR) and Hits@N metrics. %
All metric evaluations have been averaged over five runs for the countries dataset. %

\begin{figure*}[t]
    \centering
    \includegraphics[width=0.88\textwidth]{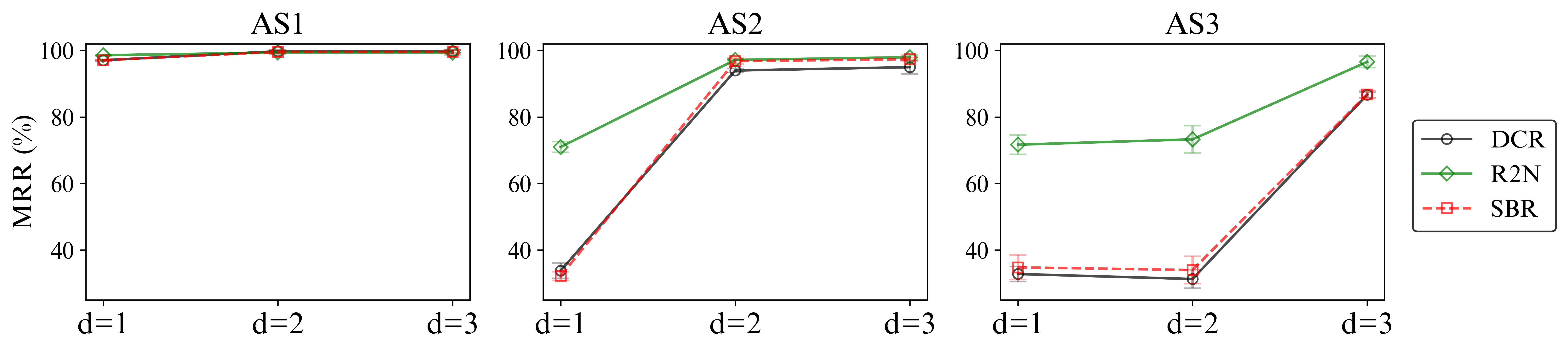}
    \caption{MRR results for different methods using $BC_{1,d}$ with different depths  for the Countries\_abl datasets $AS_1, AS_2, AS_3$.}
    \label{fig:ablation}
\end{figure*}

\begin{figure*}[th]
    \centering
\includegraphics[width=1\textwidth]{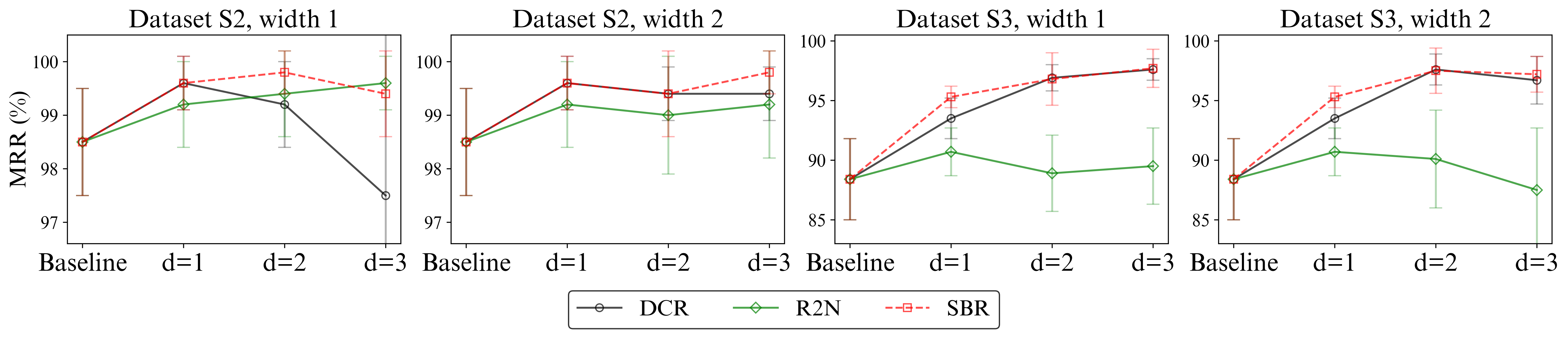}
    \caption{Results for the studied NeSy models and $BC_{w,d}$ grounders for Countries S2,S3 (S1 omitted as it is perfectly solved by all methods).}
    \label{fig:results_countries} 
\end{figure*}

\subsection{Experimental Results}
The first experiment serves as a proof of concept using the Countries\_abl dataset. Three splits ($AS_1, AS_2, AS_3$) are generated by ablating facts from the original dataset. The splits are designed with queries requiring one, two, and three reasoning steps, respectively, based on consecutive applications of the rule   $LocIn(x, w) \land LocIn(w, z) \rightarrow LocIn(x, z)$.
The $BC_{0,1}$ grounder is sufficient to solve the simple one-hop reasoning task defined by $AS_1$, whereas $BC_{1,2}$ or $BC_{1,3}$ are required to provide good MRR performances in $AS_2$, as shown in \Cref{fig:ablation}. Finally, $BC_{1,3}$ is the only grounder capable of solving $AS_3$ with high MRR with all methods. This is consistent with our expectations, as the grounding depth defines the number of hops preserved in the context of a query.

The results for the Country dataset are presented in \Cref{fig:results_countries}, where the baseline is placed as $d=0$ (this is consistent with the model definitions as the baseline is the model output if no message passing of the reasoner is performed on top of the input atom processing layer). 
The task S1 is quite simple so that all methods including the baseline are able to fully solve it with $MRR=1$, therefore it is omitted in the plots.
For S2, all NeSy methods outperform the baseline, with depth 1 ($BC_{w,1}$) sufficient for strong performance, and higher depths plateauing.
Regarding the more complex reasoning required by the S3 task, we observed that R2N tends to overfit for this dataset in the tested configuration. All other NeSy methods provide a large improvement over the baseline. Optimal results are achieved at either depth 2 or 3, which is indeed the number of reasoning hops required to solve the task.
\Cref{fig:country_full} compares the NeSy methods against the baseline when using the Full grounder for the S1,S2,S3 tasks. The performances of the Full grounder are, in general, not improving over the $BC_{w,d}$ grounders in \Cref{fig:results_countries}, with actually smaller average gains over the baseline. The Full grounder tends to ground many unnecessary atoms with respect to the focused $BC_{w,d}$ grounders, and this negatively affects model generalization (as stated by Theorem \ref{th:VC}).
The results for the WN18RR and Kinship are presented in \Cref{tab:results_other_datasets}, where we report only the grounders that could be applied in less than 10h of computation on our workstation (12 core i7 CPU, 64GB RAM, OS Linux).
For example, the application of the Full grounder is unfeasible for WN18RR or Kinship. 

\begin{figure}[t]
    \centering
    \includegraphics[width=0.40\textwidth]{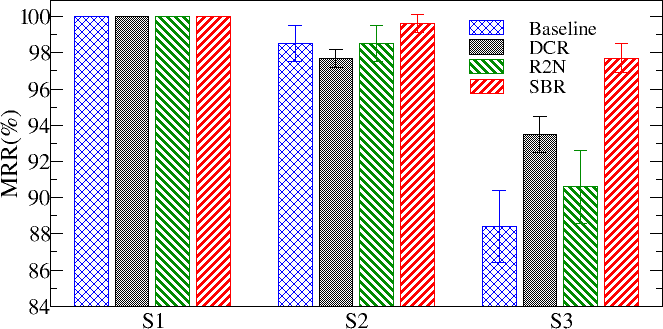}
    \caption{MRR results for the Country datasets and the Full grounder using different NeSy methods and the baseline.}
    \label{fig:country_full}
\end{figure}

\begin{table}[t]
\centering
\setlength{\tabcolsep}{2pt}
{\scriptsize
\begin{tabular}{cccccccc}
Model                         & Grounder             & MRR                            & Hits@1               & Hits@3               & Hits@10              & Train Time(s)        & Test Time(s)         \\ \hline
\multicolumn{8}{c}{\textbf{Kinship}} \\ \hline
DCR                           & $BC_{0,1}$           & 90.1                           & 84.1                 & \textbf{95.9}        & 97.0                 & 16479.6              & 7658.7               \\
                              & $BC_{1,2}$           & 90.1                           & 84.1                 & \textbf{95.9}        & 97.0                 & 16295.2              & 7517.4               \\
R2N                           & $BC_{0,1}$           & \textbf{94.0}                  & \textbf{92.1}        & 95.6                 & 96.5                 & 9573.3               & 6616.2               \\
                              & $BC_{1,2}$           & 91.8                           & 87.1                 & 96.4                 & 97.4                 & 48808.8              & 28249.2              \\
SBR                           & $BC_{0,1}$           & 86.9                           & 78.0                 & 95.6                 & 97.1                 & 9066.8               & 6208.7               \\
                              & $BC_{1,2}$           & 87.7                           & 79.1                 & 96.0                 & \textbf{97.3}        & 43354.7              & 27447.9              \\
ComplEx                       & -                    & 85.9                           & 79.2                 & 92.2                 & 94.5                 & 773.3                & 284.9                \\ \hline
\multicolumn{8}{c}{\textbf{WN18RR}} \\ \hline
DCR                           & $BC_{0,1}$           & 44.2                           & 42.2                 & 44.8                 & 47.6                 & 26133.3              & 2337.9               \\
                              & $BC_{1,2}$           & \textbf{45.6}                  & 42.9                 & \textbf{47.1}        & \textbf{50.2}        & 74626.7              & 6944.2               \\
R2N                           & $BC_{0,1}$           & 44.2                           & 42.3                 & 44.6                 & 47.3                 & 20613.7              & 2183.4               \\
                              & $BC_{1,2}$           & 44.1                           & 41.4                 & 45.4                 & 48.1                 & 72213.3              & 7353.4               \\
SBR                           & $BC_{0,1}$           & 44.0                           & 42.3                 & 44.2                 & 46.6                 & 21940.5              & 1909.8               \\
                              & $BC_{1,2}$           & 44.7                           & \textbf{42.5}        & 45.2                 & 48.2                 & 67851.6              & 6666.0               \\
ComplEx                       & -                    & 42.7                           & 40.8                 & 42.9                 & 45.9                 & 1079.2               & 138.7                \\ \hline
\end{tabular}
}
\caption{Results for different NeSy methods and grounders for the Kinship and WN18RR datasets compared against the baseline.} 
\label{tab:results_other_datasets}
\end{table}

\begin{table}[t]
\centering
\label{tab:fb15k237}
{\scriptsize
\begin{tabular}{c|ccccc}
& Grounder & MRR & Hits@1 & Hits@3 & Hits@10\\
\hline
ComplEx  & - & 25.9 & 16.3 & 29.2 & 45.2\\
RNNLogic(emb.) & $BC_{0,1}$ & 34.4 & 25.2 & 38.0 & 53.0\\
R2N    & $BC_{0,1}$ & {\bf 34.7} & {\bf 25.4} & {\bf 38.2} & {\bf 53.1}\\
\end{tabular}
}
\caption{MRR, Hits@N metrics on the FB15k-237 dataset.}
\end{table}

MRR gains over the baseline of approximately 9\%, and 3\% for Kinship, and WN18RR, respectively.
As shown by the inference times in Table \ref{tab:results_other_datasets}, increasing the depth of the grounder in the larger datasets strongly impacts the required inference times. In practice, any grounder with a depth $d>2$ would be very impractical to use for large KGs. Therefore we limited the experiments to $BC_{0,1}$ and $BC_{1,2}$. We observed generally improved results are obtained as the depth increases from $d=1$ to $d=2$ to retain enough information to improve the classification.
However, the statistical nature of NeSy tasks makes it possible to get competitive performances even when limited to $BC_{0,1}$.
For instance, in the Kinship dataset DCR works very well already at depth 1, even if R2N and SBR can provide slightly improved results at depth 2.
In the WN18RR dataset, $BC_{0,1}$ already provides significant gains over the baseline, even if $BC_{1,2}$ provides the best results both for DCR and SBR.
These differences are also found in larger datasets. For WN18RR, moving from $BC_{0,1}$ to $BC_{1,2}$ with the DCR model improved the MRR metric by 1.6 points, but increased training and inference times by a 3x factor.
This observation underscores the critical importance of efficient and effective grounding techniques for scaling NeSy methods to larger real-world KGs.
Results for the larger FB15k-237 dataset  are provided in \Cref{tab:fb15k237}, where the computational constraints limited experiments to $BC_{0,1}$, which still showed substantial improvements over the baseline for R2N. %
We did not add SBR and DCR to the table as the available rules were not covering enough atoms to provide a fair comparison. %

\section{Related Work}
\label{sec:rel_work}

\paragraph{Full Groundings Approaches.}
A large class of NeSy models was defined over a full grounding approach, this includes SBR \cite{diligenti2017semantic}, LTN~\cite{badreddine2022logic}, and Relational Neural Machines~\cite{marra2020relational}. The main problem with this grounding method is its lack of scalability, as shown in the experimental section.

\paragraph{Theorem proving approaches.}
Automatically answering a query (i.e. proving a theorem) given a set of definite clauses is at the core of logic programming~\cite{clocksin2012programming}. %
When proving a fact, only a portion of the GMN is relevant, thus a full grounding is usually not necessary. Standard techniques to avoid the full grounding include \textit{Backward Chaining}~\cite{kapoor2016comparative} and \textit{Forward Chaining}~\cite{al2015comparison} or a mix of the two \cite{mumick1994implementation}. Since known facts are in general sparse over the HB, the relevant portion of the Markov network is much smaller than the full GMN obtained by grounding the rules over the entire domain. A large class of NeSy systems is based on theorem proving, this class includes systems exploiting a Prolog engine like DeepProbLog~\cite{manhaeve2018deepproblog}, DeepStochLog~\cite{winters2022deepstochlog}, an ASP engine~\cite{shakarian2023neurasp} or Datalog (e.g. LRNN \cite{sourek2018lifted}).
However, the exact proving approach does not allow the methods to scale beyond small inference problems.
Therefore, some methods try to further limit the graph size. For example, DPLA* \cite{manhaeve2021approximate} uses an A* approach to select the most promising portions of the GMN.

\paragraph{Heuristics-based approaches.} 
The most common approach used in the literature is to rely on some ad-hoc heuristics to make grounding tractable. However, it did not emerge in the literature a clear understanding of the trade-off between performance and computing time.
\textit{RNNLogic} \cite{qu2020rnnlogic} generates logic rules and then scores the query triplets using the grounding paths in the training knowledge. This makes the model account for the grounded rules where all body atoms are true in the training facts, which is the special case where Backward Chaining is limited to depth 1.
\textit{Discriminative Gaifman Models}~\cite{niepert2016discriminative} considers local neighbourhoods of a KG as features to predict queries. The local neighbourhoods can be seen as a generalization of a Horn clause, representing a chain of facts in the knowledge that imply the query. Similarly to RNN Logic, this considers only the grounded rules where all body atoms are true in the training facts.
\textit{ExpressGNN} \cite{zhang2020efficient} operates directly on the knowledge graph rather than the extensive GMN grounding graph. This is similar to the approach of \textit{UniKER} \cite{cheng2021uniker}, which focuses on instantiating rules only with the training triples in the knowledge graph but iteratively enriching the training data with new facts derived using forward chaining.
\textit{KALE}~\cite{guo2016jointly} includes groundings where at least one atom is true in the training data, and uses a full grounding approach for the other variables. %
\textit{RUGE} \cite{guo2018knowledge} considers a grounding when the triples in the rule's body are observed in the training data, and the head triple is not.

\section{Conclusions and Future Work}
\label{sec:conc}
This paper studies the effect of different grounding schemas on the predictive accuracy of NeSy methods.
The main contributions of the work are in the definition of new parametrized grounding approaches, tailored to NeSy methods, and in the extensive experimental comparison of these grounders, offering insights into the trade-offs between grounding depth, model performance, and computational efficiency. The results on link prediction on KGs show that increasing the depth of the parameterized grounder enhances performance in complex reasoning tasks, but the statistical nature of NeSy allows preserving a good prediction accuracy even using shallow and efficient grounders.
This contribution is not only methodological but also practical, highlighting the critical role of the grounding schema in bridging the gap between symbolic reasoning and neural learning.

This study is limited to static grounders without any feedback from the NeSy method, which is the most used approach. In future work, we plan to explore dynamic grounders, which iteratively exploit the predictions of the learner to dynamically incorporate new atoms into the set of known facts.
This could lead to the definition of parametric grounders which can be trained end-to-end with the learner. %

\section*{Ethical Statement}

There are no ethical issues.

\section*{Acknowledgments}

  This work has been partially supported by the project PNRR M4C2 “FAIR - Future Artificial Intelligence Research” - Spoke 1 “Human-centered AI” , code PE0000013, CUP E53C22001610006.  This work was also supported by the EU Framework Program for Research and Innovation Horizon under the Grant Agreement No 101073307 (MSCA-DN LeMuR). This work has been partially supported by the project ‘‘CONSTR: a COllectionless-based
Neuro-Symbolic Theory for learning and Reasoning’’, PARTENARIATO
ESTESO ‘‘Future Artificial Intelligence Research - FAIR’’, SPOKE 1
‘‘Human-Centered AI’’ Università di Pisa, ‘‘NextGenerationEU’’, CUP
I53C22001380006.

\bibliographystyle{named}
\bibliography{ijcai25}

\end{document}